\documentclass[letterpaper, 10 pt, conference]{ieeeconf} 

\IEEEoverridecommandlockouts  
\usepackage{amsmath}
\usepackage{esint}
\usepackage{amssymb}
\usepackage{amsthm}
\usepackage{mathtools}
\usepackage{derivative}
\usepackage{xcolor}
\usepackage{bm}
\usepackage{mathrsfs}
\usepackage{physics} 
\usepackage[ruled,vlined]{algorithm2e}

\usepackage{graphicx}   
\usepackage{caption}    
\usepackage{subcaption} 

\usepackage[noadjust]{cite}
\usepackage{url}
\usepackage[hidelinks]{hyperref}

\usepackage{graphicx}
\usepackage[export]{adjustbox} 
\graphicspath{{figures/}}
\usepackage{stfloats}

\newtheorem{proposition}{Proposition}
\newtheorem{corollary}{Corollary}

\theoremstyle{definition}
\newtheorem{definition}{Definition}

\newtheorem{remark}{Remark}

\DeclareMathOperator*{\argmin}{arg\,min}

\title{\LARGE \bf
Risk-Aware Safety Filters with Poisson Safety Functions \\ and Laplace Guidance Fields
}

\author{Gilbert Bahati, Ryan M. Bena, Meg Wilkinson, Pol Mestres, Ryan K. Cosner, and Aaron D. Ames 
\thanks{This research is supported by BP.}
\thanks{G. Bahati, R. M. Bena, P. Mestres and A. Ames are with the Department of Mechanical and Civil Engineering, Caltech, Pasadena, CA; M. Wilkinson is with the Department of Computing and Mathematical Sciences, Caltech, Pasadena, CA. Emails: $\{$\texttt{gbahati, ryanbena, mwilkins mestres, ames}$\}$\texttt{@caltech.edu}.} 
\thanks{R. K. Cosner is with the Department of Mechanical Engineering, Tufts University, Medford, MA,
\texttt{ryan.cosner@tufts.edu}.}
}

\begin{document}

\maketitle
\thispagestyle{empty}
\pagestyle{empty}

\newcommand{\naturals}{\mathbb{N}}
\newcommand{\re}{\mathbb{R}}
\newcommand{\R}{\mathbb{R}}
\newcommand{\realnonneg}{\mathbb{R}_{\ge 0}}
\newcommand{\realpos}{\mathbb{R}_{> 0}}
\newcommand{\until}[1]{[#1]}
\newcommand{\map}[3]{#1:#2 \rightarrow #3}
\newcommand{\qedA}{~\hfill \ensuremath{\square}}
\newcommand\scalemath[2]{\scalebox{#1}{\mbox{\ensuremath{\displaystyle #2}}}}
\newcommand{\interior}{\operatorname{int}}

\newcommand{\longthmtitle}[1]{\mbox{}{\textit{(#1):}}}
\newcommand{\setdef}[2]{\{#1 \; | \; #2\}}
\newcommand{\setdefb}[2]{\big\{#1 \; | \; #2\big\}}
\newcommand{\setdefB}[2]{\Big\{#1 \; | \; #2\Big\}}
\newcommand*{\SetSuchThat}[1][]{} 
\newcommand*{\MvertSets}{%
    \renewcommand*\SetSuchThat[1][]{%
        \mathclose{}%
        \nonscript\;##1\vert\penalty\relpenalty\nonscript\;%
        \mathopen{}%
    }%
}
\MvertSets 

\newcommand{\dt}{\mathrm{d}t}
\newcommand{\dy}{\mathrm{d}y}
\newcommand{\dx}{\mathrm{d}x}
\newcommand{\dtau}{\mathrm{d}\tau}
\newcommand{\Cc}{\mathcal{C}}
\newcommand{\Ac}{\mathcal{A}}
\newcommand{\pCc}{\partial \mathcal{C}}
\newcommand{\Bc}{\mathcal{B}}
\newcommand{\Tc}{\mathcal{T}}
\newcommand{\Dc}{\mathcal{D}}
\newcommand{\Oc}{\Omega}
\newcommand{\Occ}{\overline{\Omega}}
\newcommand{\pOc}{\partial \Omega}
\newcommand{\Ocext}{\Oc_\mathrm{ext}}
\newcommand{\Ocint}{\Oc_\mathrm{int}}
\newcommand{\Hc}{\mathcal{H}}
\newcommand{\Fc}{\mathcal{F}}
\newcommand{\Mc}{\mathcal{M}}
\newcommand{\Nc}{\mathcal{N}}
\newcommand{\Pc}{\mathcal{P}}
\newcommand{\Uc}{\mathcal{U}}
\newcommand{\Sc}{\mathcal{S}}
\newcommand{\Xc}{\mathcal{X}}
\newcommand{\Yc}{\mathcal{Y}}
\newcommand{\Vc}{\mathcal{V}}
\newcommand{\Zc}{\mathcal{Z}}
\newcommand{\Lc}{\mathcal{L}}
\newcommand{\Rm}{\mathcal{\mathbb{R}}}

\newcommand{\divv}{\nabla \cdot \vec{\bv}}
\newcommand{\hs}{h_\mathrm{\Sc}}

\newcommand{\defeq}{\triangleq}

\newcommand{\vr}{\varepsilon}
\newcommand{\nom}{{\operatorname{nom}}}
\newcommand{\m}{{\operatorname{min}}}
\newcommand{\des}{{\operatorname{des}}}
\newcommand{\on}{{\operatorname{on}}}
\newcommand{\off}{{\operatorname{off}}}
\newcommand{\fl}{{\operatorname{FL}}}
\newcommand{\Lie}{\mathcal{L}}
\newcommand{\qp}{{\operatorname{QP}}}

\newcommand{\ie}{i.e., }
\newcommand{\todo}[1]{{\color{cyan} Todo: #1}}

\newcommand{\ba}{\mathbf{a}}
\newcommand{\bb}{\mathbf{b}}
\newcommand{\be}{\mathbf{e}}
\renewcommand{\bf}{\mathbf{f}} 
\newcommand{\bff}{\mathbf{f}}
\newcommand{\bg}{\mathbf{g}}
\newcommand{\bk}{\mathbf{k}}
\newcommand{\bp}{\mathbf{p}}
\newcommand{\bq}{\mathbf{q}}
\newcommand{\bu}{\mathbf{u}}
\newcommand{\bv}{\mathbf{v}}
\newcommand{\bvv}{\vec{\mathbf{v}}}
\newcommand{\bn}{\mathbf{n}}
\newcommand{\hbn}{\hat{\mathbf{n}}}

\newcommand{\bx}{\mathbf{x}}
\newcommand{\bz}{\mathbf{z}}
\newcommand{\br}{\mathbf{r}}
\newcommand{\bA}{\mathbf{A}}
\newcommand{\bB}{\mathbf{B}}
\newcommand{\bD}{\mathbf{D}}
\newcommand{\bC}{\mathbf{C}}
\newcommand{\bF}{\mathbf{F}}
\newcommand{\bJ}{\mathbf{J}}
\newcommand{\bG}{\mathbf{G}}
\newcommand{\bK}{\mathbf{K}}
\newcommand{\bP}{\mathbf{P}}
\newcommand{\bW}{\mathbf{W}}
\newcommand{\bw}{\mathbf{w}}
\newcommand{\bd}{\mathbf{d}}
\newcommand{\bvy}{\vec{\by}}
\newcommand{\bty}{\tilde{\by}}
\newcommand{\bbeta}{\boldsymbol{\eta}}
\newcommand{\mb}[1]{\mathbf{#1}}

\newcommand{\bY}{\mathbf{Y}}
\newcommand{\by}{\mathbf{y}}
\newcommand{\byobs}{\mathbf{y}_\mathrm{obs}}
\newcommand{\bl}{\mathbf{\lambda}}

\newcommand{\bxd}{\bx_\mathrm{d}}
\newcommand{\bxobs}{\bx_\mathrm{obs}}
\newcommand{\md}{\mathrm{d}}

\newcommand{\Uxd}{U_{\mathrm{d}}}
\newcommand{\Uobs}{U_{\mathrm{obs}}}
\newcommand{\Uapf}{U_{\mathrm{APF}}}

\newcommand{\GradUxd}{\nabla U_{\mathrm{d}}}
\newcommand{\GradUobs}{\nabla U_{\mathrm{obs}}}
\newcommand{\GradUapf}{\nabla U_{\mathrm{APF}}}

\newcommand{\cmax}{c_\mathrm{max}}
\newcommand{\cmin}{c_\mathrm{min}}

\newcommand{\hn}{h_\mathrm{n}}
\newcommand{\Dhn}{D h_\mathrm{n}}
\newcommand{\Dh}{D h}
\newcommand{\Dhd}{D h_\mathrm{d}}

\begin{abstract}
    %

%
Robotic systems navigating in real-world settings require a semantic understanding of their environment to properly determine safe actions.  
This work aims to develop the mathematical underpinnings of such a representation--specifically, the goal is to develop safety filters that are risk-aware.  
To this end, we take a two step approach: encoding an understanding of the environment via Poisson's equation, and associated risk via Laplace guidance fields.  
That is, we first solve a Dirichlet problem for Poisson's equation to generate a safety function that encodes system safety as its 0-superlevel set. 
We then separately solve a Dirichlet problem for Laplace's equation to synthesize a safe \textit{guidance field} that encodes variable levels of caution around obstacles---by enforcing a tunable flux boundary condition. 
The safety function and guidance fields are then combined to define a safety constraint and used to synthesize a risk-aware safety filter which, given a semantic understanding of an environment with associated risk levels of environmental features, guarantees safety while prioritizing avoidance of higher risk obstacles. 
We demonstrate this method in simulation and discuss how \textit{a priori} understandings of obstacle risk can be directly incorporated into the safety filter to generate safe behaviors that are risk-aware. 
\end{abstract}

\section{Introduction}

As modern robots increasingly venture into the real world, they encounter obstacles with variable degrees of relevance to system safety. Different obstacle are often associated with different levels of risk, motivating different degrees of conservatism during navigation.
For example, while avoiding collisions in the environment, it is often important for the robot to behave more cautiously around high-risk obstacles like humans or expensive equipment. A \emph{risk-aware} safety approach ensures collision avoidance while incorporating additional conservatism near high-risk obstacles.  

One common approach for encoding system safety is through control barrier functions (CBFs)\cite{AmesTAC17, AA-AJT-CRH-GO-AA:22}. CBFs divide the system's operable region into safe states and unsafe states and can be used to synthesize safe controllers by enforcing the forward invariance of the safe region. However, this binary representation of safety, especially in the context of collision avoidance, often treats all possible collisions equally by imposing uniform gradients on the boundary of the safe region \cite{long2021learning,AmesECC19}. This uniformity limits the ability to tailor the system behavior to specific obstacles or spatially-dependent risk factors. 
\begin{figure}[t!]
    \centering    \includegraphics[width=1\linewidth]{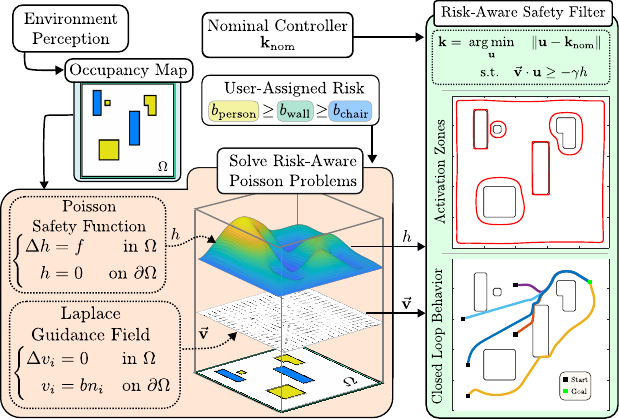}
\caption{\small{The proposed risk-aware control synthesis method. By incorporating user-assigned risk values in the guidance field, our approach generates controllers and closed-loop behaviors exhibiting conservatism aligned with the prescribed risk levels.}}
    \label{fig:hero_figure}
    \vspace{-7mm}
\end{figure}

While CBFs have been shown to be powerful theoretical tools for defining and enforcing safety, they are often difficult to generate, even for systems with simple dynamics. 
For complex environments, CBFs are often constructed \textit{ad hoc} with limited generalizability \cite{TM-RC-AS-WU-AA:22,cosner2022safety, glotfelter2018boolean}. While some reachability and learning approaches have shown promise as generalizable methods for automatic CBF synthesis \cite{choi2021robust,long2021learning,qin2021learning,robey2020learning}, they are computationally complex, and typically require that the CBF be computed offline, preventing actively incorporation of dynamic environment data.
Alternatively, Poisson safety functions (PSFs) \cite{bahati2025dynamic} enable real-time generation of safe sets from perception data by solving a Dirichlet problem for Poisson's equation online.
This approach leverages a \textit{guidance field} that encodes gradient information required for safety \cite{bahatisafe}. 
The guidance field provides additional flexibility in defining safety by allowing boundary conditions to specify desired gradients (i.e., boundary flux) on obstacle surfaces.

Inspired by the ability to assign the desired  boundary flux for PSFs, this work presents a method for generating safe behavior with variable, spatially-dependent, user-assigned conservatism by directly assigning boundary flux values through \textit{Laplace} guidance fields and decoupling the safety gradient from the safety value in the standard CBF safety filter \cite{ADA-XX-JWG-PT:17}. Specifically, the safety gradient, \ie the vector field yielding system safety, is now directly defined by the guidance field and is not necessarily the exact gradient of the PSF. While other methods such as tunable input-to-state safe CBFs \cite{AnilLCSS22} and state-dependent CBF decay conditions \cite{cosner2023learning} also enable spatially variable degrees of conservatism, neither are capable of rapidly generating CBFs with variable conservatism online.
Alternatively, our method is constructive, capable of online synthesis, and results in behaviors which display different levels of caution depending around different obstacles depending on the assigned risk---remaining agnostic to what the risk represents and how it is quantified \cite{akella2024risk}, thus is broadly applicable to different scenarios and objectives.

The main contributions of this work are threefold. First, we develop a novel method for automatically synthesizing PSF-based safety filters with Laplace guidance fields to incorporate \textit{a priori} risk-awareness, achieving spatially variable conservatism. Next, we provide analysis on how adjustments of the guidance field affect the \textit{activation zones} of the safety filter, directly displaying how modifications to the flux impact the conservatism of the filter. Finally, we provide several examples demonstrating how our method can be used to achieve risk-aware safety in a variety of contexts such as environments with probabilistic occupancies, dynamic obstacles or semantic labels with risk magnitudes.

\section{background}
\noindent
Consider the nonlinear control affine system of the form:
\begin{align}\label{eq: nl sys}
    \dot{\bx} = \bf(\bx) + \bg(\bx)\bu,
\end{align}
where $\bx \in \re^n, \bu \in \re^m$ are the state and input, and $\bf:\re^n \rightarrow \re^n, \bg:\re^{n}\rightarrow \re^{n \times m} $ are assumed to be locally Lipschitz continuous functions. Given a locally Lipschitz continuous controller $\bk:\re^n \rightarrow \re^m$ , the closed-loop system $\dot{\bx} = \bf_{\mathrm{cl}}(\bx) = \bf(\bx) + \bg(\bx)\bk(\bx)$ has a unique solution for any initial condition $\bx_0\in\mathbb{R}^n$. For all closed-loop systems considered in this work, we assume that such solutions exist for all $t \geq 0$ for ease of exposition.

\subsection{Safety and Control Barrier Functions}
We formalize the definition of safety as the forward invariance of a \textit{safe set}, where the system is considered safe when all trajectories of the closed-loop system remain in the desired safe set for all $t \geq 0$. In particular, we consider safe sets defined as the $0$-superlevel set of a continuously differentiable function $\hs:\re^n\rightarrow \re$  as:
\begin{align}
    \label{eq: safe set}
    \Sc = \left\{\bx \in \re^n \, \big| \, \hs(\bx) \geq 0\right\}.
\end{align}%
Control Barrier Functions (CBFs) are a constructive tool that can be used to design controllers for \eqref{eq: nl sys} that enforce the forward invariance of the set $\Sc$.

\begin{definition}(Control Barrier Functions \cite{AmesTAC17})
We call a function $\hs:\re^n\rightarrow \re$ a Control Barrier Function (CBF) for \eqref{eq: nl sys} if there exists\footnote{ A continuous function $\gamma: \mathbb{R} \rightarrow \mathbb{R}$ is an \textit{extended class} $\mathcal{K}$, denoted by $\gamma \in \mathcal{K}^{e}_{\infty}$, if $\gamma$ is monotonically increasing, $\gamma(0) = 0$, $\lim_{s \rightarrow \infty} \gamma(s) = \infty$, and $\lim_{s \rightarrow -\infty} \gamma(s) = -\infty$.} $\gamma \in \mathcal{K}^e_\infty$ such that for all $\bx \in \re^n$, the following condition holds:
    \begin{align}\label{CBF Condition}\!\!\!
        \sup_{\bu \in \re^m} \! \! \big\{ \underbrace{D \hs (\bx) \!\cdot\! \bf(\bx)}_{L_\bf \hs(\bx)} + \underbrace{D \hs (\bx) \!\cdot\! \bg(\bx)}_{L_\bg \hs(\bx)}\bu \big \} \!>\! -\gamma(\hs(\bx)),
    \end{align}
    where $D h_{\mathcal{S}}$ denotes the gradient of $h_{\mathcal{S}}$.
\end{definition} 
%
Given a nominal controller $\bk_\nom:\re^n \rightarrow \re^m$, and a CBF $h_\Sc$, a typical way of synthesizing safe controllers is through quadratic programming-based safety filters, which adjust $\bk_\nom$ to the nearest safe action:
\begin{align*}\label{eq: safety filter}
    \bk(\bx) = &\argmin_{\bu \in \re^m} &&\|\bu - \bk_{\mathrm{nom}}(\bx)\|_2^2 \tag{Safety-Filter} \\
    & \quad 
 \mathrm{s.t.} &&L_\bf \hs(\bx) + L_\bg \hs(\bx)\bu  \geq - \gamma(\hs(\bx)).
\end{align*}
%
%
\subsection{Outputs and Relative Degree}
This manuscript considers safety specifications which can be represented via a set of desired \textit{outputs}. We recall the notion of \textit{relative degree}, which describes the level of differentiation at which a control input affects an output.
\begin{definition}[Relative Degree $r$ \cite{isidori1985nonlinear}]\label{def:relative-degree}
    A function $\by\,:\,\R^n\rightarrow\R^p$ has \emph{relative degree} $r\in\mathbb{N}$ for \eqref{eq: nl sys} if:
    \begin{align}
        L_{\bg}L_{\bf}^{i}\by(\bx) \equiv \mathbf{0}, &\quad \forall i\in\{0,\dots,r-2\},\label{eq:zero-lie-derivatives} \\
        \mathrm{rank}(L_{\bg}L_{\bf}^{r-1}\by(\bx)) = p, & \quad \forall \bx\in\R^n.\label{eq:decoupling-full-rank}
    \end{align}
\end{definition}
\noindent
Given an output $\by$ with relative degree $r$, we define a new set of partial coordinates:
\begin{equation}\label{eq:output-coordinates}
      \vec{\mb{y}}(\bx) \coloneqq \begin{bmatrix}
        \by(\bx) \\
        \by^{(1)}(\bx) \\
        \vdots \\
        \by^{(r-1)}(\bx)
    \end{bmatrix}
    =
    \begin{bmatrix}
        \by(\bx) \\
        L_{\bf}\by(\bx) \\
        \vdots \\
        L_{\bf}^{r-1}\by(\bx)
    \end{bmatrix}
    \in\re^{pr},
\end{equation}
where $\by^{(r)} = \dv[r]{\by}{t}$, leading to the following linear dynamics:
\begin{align}\label{eq: linear output dynamcis1}
    \frac{\mathrm{d}}{\dt}\vec{\mb{y}}(\bx) & = 
    \underbrace{
    \begin{bmatrix}
        \mb{0} & \mb{I}_{p(r-1)}\\
        \mb{0} & \mb{0}
    \end{bmatrix}}_{\mb{A}} \vec{\mb{y}}(\bx)
    \! + \!
    \underbrace{
    \begin{bmatrix}
        \mb{0} \\
        \mb{I}_p
    \end{bmatrix}}_{\mb{B}}
    \bw\\
    \bw &\coloneqq L_{\bf}^{r}\by(\bx) + L_{\bg}L_{\bf}^{r-1}\by(\bx)\bu,     \label{eq: y_r mapping}
\end{align}
where \eqref{eq: y_r mapping} is an input to \eqref{eq: linear output dynamcis1}. When $\by$ has relative degree $r$, the controller $\bw= \hat{\bk}(\vec{\by})$ designed for \eqref{eq: linear output dynamcis1} can be transferred back to the controller for \eqref{eq: nl sys} as follows\footnote{Condition \eqref{eq:decoupling-full-rank} implies the right psuedo-inverse $L_{\bg}L_{\bf}^{r-1}\by(\bx)^\dagger$ exists.}:
\begin{equation}\label{eq:input-transformation}
    \bu = L_{\bg}L_{\bf}^{r-1}\by(\bx)^\dagger\left[\hat{\bk}(\vec{\by}(\bx)) - L_{\bf}^{r}\by(\bx) \right].
\end{equation}
This partial coordinate transformation is a full coordinate transformation if $pr=n$.
The assumption that $\by$ has relative degree $r$, and the form of the output dynamics \eqref{eq: linear output dynamcis1},
enables the employment of various CBF construction methods for the original dynamical system~\eqref{eq: nl sys}~\cite{cohen2024constructive,bahati2025control,AndrewCDC22,MurrayACC20}. Next, we discuss a method of synthesizing CBFs for environmentally relevant safety specifications, as presented in \cite{bahati2025dynamic}.
%
\subsection{Poisson Safety Functions}
We focus on systems for which safety specifications are described in spatial coordinates $\by =(x,y,z)\in \re^3$. 
Given environmental occupancy data, let $\Oc$ be a smooth, open, bounded and connected set representing unoccupied regions and $\pOc$ represent the surfaces of occupied regions.
Specifically, $\pOc = \bigcup_{i}^{n_\mathrm{o}} \partial \Gamma_i$
 where $\Gamma_i$ is an open, bounded and connected set corresponding to the interior of an occupied region with $n_\mathrm{o}$ denoting the total number of occupied regions.
A safety function provides a functional representation of safety for an environment, defined as follows.

\begin{definition}(Safety Function \cite{bahati2025dynamic})\label{def: safety func} Let $\by = (x,y,z) \in \re^3$ represent coordinates in three dimensional space. 
We call a function $h:\Occ \rightarrow \re$ a safety function of order $k$ on $\Occ$ if 
$h$ is $k$-times differentiable, $Dh(\by) \neq 0$ when $h(\by) = 0$ , and the $0$-superlevel set of $h$ characterizes a safe set:
\begin{subequations}
\begin{align}\label{eq:safe set poisson}
   \Cc = \{\by \in \Occ: h(\by) \geq 0\},\\
      \partial \Cc = \{\by \in \Occ: h(\by) =0\},\\
          \mathrm{Int}(\Cc) = \{\by \in \Occ: h(\by) >0\}.
\end{align}
\end{subequations}
\end{definition}
%
%
Given environmental data characterizing the domain $\Occ$ through an occupancy map, Poisson safety functions \cite{bahati2025dynamic} generate safe sets satisfying Def.~\ref{def: safety func} by solving a Dirichlet problem for Poisson's equation:
\begin{gather}\label{eq: poisson's eq}
\left \{
    \begin{aligned}
        \Delta h(\by) &= f(\by)& \text{ in } \Omega,\\
        h(\by) &= 0 &  \text{ on } \partial \Omega, \\
    \end{aligned}
    \right.
\end{gather}
where \( \Delta = \frac{\partial^2 }{\partial x^2} + \frac{\partial^2 }{\partial y^2} + \frac{\partial^2 }{\partial z^2}\) is the \textit{Laplacian} and $f: \Oc \rightarrow \re_{<0}$ is a given forcing function.
As discussed in \cite{gilbarg1977elliptic}, under appropriate regularity assumptions on $\Oc$, a smooth forcing function  $f \in C^\infty(\Occ)$ yields a smooth solution $h \in C^\infty(\Occ)$ to \eqref{eq: poisson's eq}. As demonstrated in \cite{bahati2025dynamic}, this smooth solution characterizes the safe set $\Cc$ such that $\Oc = \mathrm{Int}(\Cc), \partial \Cc = \pOc$, and
may be used to construct safety filters yielding safe control actions for \eqref{eq: nl sys} under appropriate relative degree assumptions. 

\subsection{Boundary Flux}
Given an occupancy map, $\Occ$, where $\hbn:\pOc \to \re^3$ denotes the outward pointing unit normal, collision avoidance applications require a negative outward directional derivative on the boundary, 
\ie negative \textit{boundary flux} $Dh (\by) \cdot \hat{\bn} (\by) < 0$ for all $\by \in \pOc$, to encode repulsive gradients on obstacle surfaces. From Hopf's Lemma \cite{protter2012maximum}, it follows that solving the Dirichlet problem \eqref{eq: poisson's eq} with a negative forcing function in the interior, \ie $f(\by) <0$ for all $\by \in \Oc$, guarantees $Dh (\by) \cdot \hat{\bn} (\by) < 0$ for all $\by \in \pOc$. 

The ability to prescribe the magnitude of $Dh (\by) \cdot \hat{\bn} (\by)$ at each point $\by \in \pOc$ provides a way to locally encode desired gradient strengths along obstacle surfaces. This enables the ability to prescribe stronger or weaker repulsive effects, depending on the relative importance of an obstacle (or portion of the obstacle). However, \eqref{eq: poisson's eq} does not provide the ability to prescribe $Dh (\by) \!\cdot \!\hat{\bn} (\by)$ directly in a point-wise manner.
Instead, \eqref{eq: poisson's eq} constrains the boundary flux in an integral sense through the forcing function $f$. In particular, by the divergence theorem~\cite{marsden2003vector}, all forcing functions satisfy (dropping dependency on $\by$ for brevity):
\begin{align}\label{eq: divergence forcing function}
\iiint_\Oc f \, \mathrm{d} E =  \iiint_\Oc \overbrace{\nabla \cdot (Dh)}^{\Delta h} \, \mathrm{d} E = \overbrace{\oiint_{\pOc} Dh \cdot \hat{\bn}\, \mathrm{d} A}^{\text{Total Flux}}, 
\end{align}
where $\mathrm{d} E$, $\mathrm{d} A$ denote volume and area elements respectively. 

One approach of encoding desired point-wise boundary flux magnitudes, proposed in \cite{bahati2025dynamic}, was to introduce an auxiliary vector field—a \textit{guidance field}.
In that formulation, the divergence of the guidance field served as a tool for defining the forcing function $f$. 
However, the resulting desired point-wise flux condition also held only in the integral sense via a variational problem.
Moreover, the role of the guidance field was left implicit, treated merely as an axillary tool for defining a forcing function rather than a central object of study.
In this work, we make the guidance field explicit and study its central role in enforcing safety specifications directly. We introduce a definition that captures the minimal requirements that such a field must satisfy to serve as a foundation for more expressive notions of safety, enabling obstacle-specific safety behaviors.
%

\section{Risk-Aware Safety-Critical Control Using Guidance Fields}

In this section, we present a new safety constraint that leverages \textit{guidance fields} to enforce distinct gradient behavior across domain boundaries. This provides a way to assign relative importance to boundary regions, capturing obstacle risk and priorities.
%
We begin by formalizing the notion of a guidance field. Intuitively, a guidance field is a vector field that prescribes repulsive directions along obstacle surfaces, and extends these directions smoothly into the domain. 

\begin{definition}[Guidance Field]\label{def: guidance field}
Let \(\Oc \subset \mathbb{R}^3\) be an open, bounded, and connected set representing free space with smooth boundary \(\pOc\) corresponding to obstacle surfaces, and let $\hbn:\pOc \to \re^3$ denote the unit normal pointing outward from $\Oc$, \ie into the obstacles. 
Given a prescribed negative boundary flux \(b:\pOc \to \mathbb{R}_{<0}\), we call a vector field \(\bvv \in C^{k}(\overline{\Oc};\mathbb{R}^3)\) with \(k \ge 1\) a \emph{guidance field} if:
\begin{align}\label{eq: def condition}
    \bvv(\by)\cdot \hat{\bn}(\by) = b(\by), \, \quad \, \bvv(\by)\parallel\hat{\bn}(\by) \quad \text{on }\pOc,
\end{align}
and $\bvv$ is a $C^k$ extension of the flux into the interior $\Oc$. 
\end{definition}

This definition ties the guidance field to obstacle boundaries through the flux $\bvv(\by) = b(\by)\hbn(\by)$ on $\pOc$, with negative values of $b$ encoding repulsive gradients. 
For well-posedness, $b$ must be sufficiently regular to admit a continuous extension into the domain. 
The $C^k$ regularity condition ensures that derivatives are well defined in the classical sense for the divergence theorem to hold, and to provide continuous derivatives that are convenient for control design. 
An approach for constructing guidance fields satisfying Def. \ref{def: guidance field}, is based on the \textit{vector} Laplace equation  proposed in \cite{bahati2025dynamic}.

\subsection{Laplace Guidance Field}
In the vector Laplace formulation, each component of $\bvv$ is obtained by a harmonic extension of the boundary data. This produces a smooth interpolation of the boundary flux throughout the domain in each direction, yielding a smooth guidance field satisfying Def. \ref{def: guidance field}.
Specifically, consider $\vec{\bv} = (v_x, v_y, v_z) :\overline{\Oc} \rightarrow \re^3$, with each component satisfying
Laplace's equation subject to Dirichlet boundary conditions:
\begin{gather}\label{eq: guidance field}
\left \{
    \begin{aligned}
    \Delta v_i(\by) &= 0 & \text{in } \Omega, \\
    v_i(\by) &= b(\by) n_i(\by) & \text{on } \partial \Omega,
    \end{aligned}
    \right.
\end{gather}
for $i \in \{x, y, z\}$, where \( \hat{\mathbf{n}} = (n_x, n_y, n_z):\partial \Oc \rightarrow \re^3  \) denotes the outward unit normal vector such that $\bvv(\by) = b(\by) \hat{\bn}(\by)$  on \( \partial \Omega \), and
\( b : \pOc \to \mathbb{R}_{<0} \) prescribes the outward directional derivative encoding the desired boundary flux encoding repulsive gradients.
Although \eqref{eq: guidance field} produces a smooth vector field $\bvv \in C^\infty(\Occ; \re^3)$ that matches the boundary specifications \eqref{eq: def condition}, 
the decoupled nature of the components of $\bvv$ in the vector Laplace formulation \eqref{eq: guidance field} makes the field generally \textit{nonconservative}, meaning it is not the gradient of a potential \cite{marsden2003vector}. However, since $\bvv$ satisfies \eqref{eq: def condition} by construction, it can be directly used to enforce safety without the need of a potential, which we discuss next. 

\subsection{Risk-Aware Safety Filters}
Given a safety function $h$ defining a safe set $\Cc$ as in Def.\;\ref{def: safety func}, the classical CBF condition \eqref{CBF Condition} uses the gradient $Dh$ to enforce safety. However, directly prescribing desired gradient magnitudes along $\partial \Cc$ with $Dh$ is typically not possible. We generalize the classical CBF formulation  \eqref{CBF Condition} by decoupling the vector field used in the gradient term from the function $h$.
%
In particular, we introduce the guidance field $\bvv$, which enables local, pointwise design of boundary-normal gradients (\ie boundary flux) while maintaining safety. In this sense, the CBF condition \eqref{CBF Condition} is reformulated with $\bvv$, opening new directions for safe vector-field generation.

We focus on systems defined by integrator chains as in \eqref{eq: linear output dynamcis1}, with the input appearing at the last layer; note, however, that our method can be extended to classes of systems with outputs of non-uniform relative degree \cite{cohen2024constructive, bahati2025control}. 
To formalize this, we begin with first-order systems.

\subsubsection{First Order Systems}
Consider the single integrator dynamics (relative degree $r=1$):
 \begin{align}\label{eq: single integrator}
     \dot{\vec{\by}} = \bw,
 \end{align}
 where the state $\vec{\by} = \by \in \re^3$. Given a guidance field $\bvv$, the following proposition establishes safety for \eqref{eq: single integrator}.

 \begin{proposition}(Forward Invariance of First Order Systems)\label{prop: forward invariance of first order systems}
 \label{prop: static constraint}
    Let $\Omega \subset \re^3$ be an open, bounded, and connected set with smooth boundary $\pOc$ and outward pointing normal $\hat{\bn}: \pOc \rightarrow \re^3$. Consider the system  \eqref{eq: single integrator} and a safe set, $\Cc$, defined as the $0$-super-level set of a safety function $h:\Occ \rightarrow \re$ as in Def. \ref{def: safety func} and satisfies $h(\by) = 0$ on $\partial \Cc = \pOc$. Suppose that $\bvv \in C^1(\Occ;\re^3)$ is a vector field satisfying $\bvv(\by) = b(\by) \hat{\bn}(\by)$ on $\pOc$ for a negative boundary flux $b:\pOc \rightarrow \re_{<0}$ as in Def. \ref{def: guidance field}, then for any locally Lipschitz continuous controller $\bk : \Occ \rightarrow \re^3$ satisfying: 
%
    \begin{align}\label{eq: proposition 1}
        \bvv(\by) \cdot \bk(\by)  \geq - \gamma h(\by) \quad \forall \, \by \in \Cc,
    \end{align}
   for some $\gamma > 0$, the set $\Cc$ is rendered forward invariant.
 \end{proposition} 

\begin{proof}
     Since we have that $\bvv(\by) = b(\by) \hat{\bn}(\by)$ on $\pOc$ and $h(\by)=0$ on $ \partial \Cc = \partial \Oc$, then there exists a controller $\bk$ enforcing $\bvv(\by) \cdot \bk(\by) = b(\by) \hat{\bn}(\by) \cdot  \bk(\by) \geq 0$ for all $\by \in  \partial \Cc$. Since $\pOc$ is the $0$-level set of $h$, we have:
     \begin{align}
         \hbn(\by) = c(\by)\frac{Dh (\by)}{\|Dh(\by)\|},
     \end{align}
     where $c:\pOc \rightarrow \re_{<0}$ from Hopf's Lemma\footnote{Generally $c(\by)\!\in\!\{+1,-1\}$; if $h(\by)>0$ in $\Oc$, then $c(\by)\!\equiv\! -1$ on $\pOc$.} \cite{protter2012maximum, bahati2025dynamic}.
     Thus, the enforced inequality can be rewritten as: 
     \begin{align}
         b(\mb{y})\hat{\mb{n}}(\mb{y}) \cdot \mb{k}(\mb{y}) = b(\mb{y})\frac{c(\mb{y})}{\Vert D h(\mb{y})\Vert} Dh(\mb{y}) \cdot \mb{k}(\mb{y}) \geq 0, \\
         \implies \dot{h}(\mb{y}) = Dh(\mb{y}) \cdot \mb{k}(\mb{y}) \geq 0,\label{eq: final nagumo's}
    \end{align}
    for all $\mb{y} \in \partial \Omega$, where the implication follows from the strict negativity of $b(\mb{y})$ and $c(\mb{y})\equiv -1$, the fact that $\Vert Dh(\mb{y})\Vert>0$ on $\partial \Omega$, and the single integrator dynamics \eqref{eq: single integrator}. 
     In particular, \cite{CohenLCSS23} provides examples of a locally Lipschitz continuous controllers satisfying \eqref{eq: final nagumo's}.
 Therefore, from Nagumo's theorem, the set $\Cc$ is a rendered forward invariant for the system \eqref{eq: single integrator}.
\end{proof}
The above proposition states that if the guidance field $\bvv$ is normal to $\pOc$ towards the interior $\Oc$, then a controller enforcing forward invariance of the safe set $\Cc$ exists. 
Given a nominal controller $\bk_\nom:\Occ \rightarrow \re^3$, a safety function $h$, and a guidance field $\bvv$ with negative boundary flux,
one example of a safe controller is the QP-based safety filter:
\begin{align}
    \bk_\mathrm{QP}(\by) = &\argmin_{\bw \in \re^3} \quad \|\bw - \bk_{\nom}(\by)\|_2^2 \label{eq: risk-aware safety filter}
    \\
    & \quad 
 \mathrm{s.t.} \quad \bvv (\by) \cdot \bw  \geq - \gamma h(\by), \label{eq: safety filter single int constraint}
\end{align}
whose closed-form expression is:
%
%
\begin{align}\label{eq: k_qp}
    \bk_{\mathrm{QP}}(\by) = \bk_{\mathrm{nom}}(\by) + \frac{\mathrm{ReLU}(-a(\by))}{\|\bvv(\by)\|^2}\,\bvv(\by),
\end{align}
where $\mathrm{ReLU}(-a(\by)) \coloneqq \max\{0,-a(\by)\}$ is an \textit{activation function} with $a: \Occ \rightarrow \re$ given by:
\begin{align}\label{eq: activation}
    a(\by) := \,\bvv(\by) \cdot \bk_{\mathrm{nom}}(\by) + \gamma h(\by).  
\end{align}
The controller \eqref{eq: k_qp} modifies the nominal input $\bk_{\mathrm{nom}}$ in the direction of $\bvv$ in a minimally invasive fashion whenever $\bk_{\mathrm{nom}}$ violates the CBF constraint \eqref{eq: safety filter single int constraint}. 
Specifically, the second term on the right-hand side of \eqref{eq: k_qp} is a correction term, and the role of $\mathrm{ReLU}(-a(\by))$ is to activate the modification only when $a(\by) \leq 0$. Concretely, if $a(\by) > 0$, then $\bk_{\mathrm{nom}}$ satisfies the constraint and the filter remains inactive, and if $a(\by) \leq 0$, the filter is active and the nominal controller is modified along $\bvv(\by)$.  
This implies that the filter is triggered into activation exactly on the $0$-level set of $a$, i.e.\ when $a(\by)=0$. 
We define these regions 
as \textit{activation zones}, whose size and shape are dictated by $\bk_{\mathrm{nom}}$, $\bvv$, $h$ and $\gamma$ as follows:

\begin{definition}[Activation Zone]
Let $\bk_{\mathrm{nom}}:\Occ \to \re^3$ be a nominal controller, $\bvv:\Occ \to \re^3$ a guidance field, $h:\Occ \to \re$ a safety function and $\gamma >0$.
The \textit{activation zone} is the set:
\begin{align}\label{eq: activation set}
    \Ac \coloneqq \{ \, \by \in \Occ \mid a(\by) = 0 \, \},
\end{align}
where $a(\by)$ is defined in \eqref{eq: activation}.
\end{definition}

\begin{figure}[t!]
\centering
\includegraphics[width=1\linewidth]{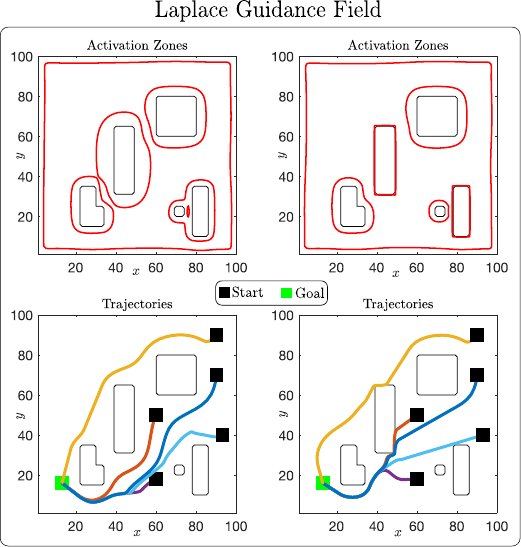}
\caption{\small{Comparison of activation zones \eqref{eq: activation} for varying boundary flux with, decreasing the flux for the center and bottom right obstacles. Reducing $\lVert b(\by) \rVert$ yields shallower boundary gradients and smaller activation zones,
allowing trajectories closer to the obstacle, whereas larger $\lVert b(\by) \rVert$ produces steeper gradients, larger activation zones and more conservative behavior. \textbf{Top row:} Activation zones \eqref{eq: activation set} with $\bk_{\mathrm{nom}}(\by)\!=\!-\mu Dh(\by), \mu\!>\!0$, driving in the direction of steepest decrease of $h$ (\ie the worst-case, adversarial direction).
\textbf{Bottom row:} Closed-loop trajectories for \eqref{eq: single integrator} with \eqref{eq: k_qp} and $\bk_{\mathrm{nom}}(\by)\!=\!-\mu(\by-\by_{\mathrm d})$, driving towards the goal $\by_{\mathrm d}$.
}}
\label{fig: activation_zones}
\vspace{-4mm}
\end{figure}
Activation zones reveal the regions where the safety filter responds to each obstacle, resulting in distinctive behaviors around individual obstacles; a larger activation zone means that the filter reacts to that obstacle sooner.
Adjusting $\gamma$ uniformly expands or contracts all activation zones in the same way, resulting in a global effect.
In contrast, isolated behavior around obstacles can be achieved by using a guidance field that enforces a pointwise boundary flux $b(\by)$ at each $ \by \in \pOc$. 
Figure~\ref{fig: activation_zones} shows how varying  the boundary flux $b$ in the guidance field $\bvv$ generated by the Dirichlet problem for Laplace \eqref{eq: guidance field} produces different activation zones.

The above result can be extended to facilitate dynamic environments with moving obstacles \cite{bena2025geometry}. The following corollary establishes the forward invariance for the single integrator system \eqref{eq: single integrator} for changing environments.

\begin{corollary}(Forward Invariance of Time-Varying Safe Sets)
    Let the assumptions of Proposition~\ref{prop: forward invariance of first order systems} hold. 
    Let $h:\re^3 \times [0, T] \rightarrow \re_{\geq 0}$ be continuously differentiable in $t$ for each $\by$, \ie $h(\by, \cdot) \in C^1([0,T])$ for some $T>0$, such that the time-varying $0$-superlevel set of $h$ characterizes the safe set:
    \begin{align}
        \Cc_\mathcal{T}(t) = \{\,\by \in \re^3 \mid h(\by,t) \geq 0 \,\}, \quad t \in [0,T],
    \end{align} 
    with $h(\by,t)=0$ on $\partial \Cc_{\Tc}(t)$. 
    Let $\sigma: \re_{\geq 0} \rightarrow [0, \epsilon]$ be a smooth transition function, for some $\epsilon > 0$ satisfying $\sigma(0) = 0$ and  $\sigma(a) \rightarrow \epsilon$ as $a \rightarrow \infty$.
    %
    %
     If there exists a locally Lipschitz continuous controller $\bk : \re^3 \to \re^3$ such that, for any $\gamma > 0$ and all $(\by, t) \in \re^3 \times [0, T]$: 
     %
\begin{equation}\label{eq: dynamic v}
    \bvv(\by)\cdot \bk(\by) +
    \frac{\lVert \bvv(\by)\rVert}{\lVert Dh(\by,t)\rVert + \sigma(h(\by,t))}
    \cdot \frac{\partial h}{\partial t}(\by,t)
    \geq -\gamma h(\by,t),
\end{equation}
then the set $\Cc_\mathcal{T}(t)$ is rendered forward invariant $\forall t \in [0,T]$.
\end{corollary}
\begin{proof}
    The result follows by applying the proof of Proposition~\ref{prop: static constraint} to the time-varying case, noting that the coefficient of  $\frac{\partial{h}}{\partial{t}}$ ensures $\dot{h}(\by,t) \geq 0$ on the boundary. That is, on boundary, since we have have that $\bvv(\by) \parallel Dh (\by, t)$, we have (dropping dependency on $(\by, t)$ for brevity):
    \begin{align}
        \bvv = \frac{\|\bvv\|}{\|Dh\|}Dh \text{ on } \pCc_{\Tc}(t),
    \end{align}
    and $\sigma(0) = 0$. Then \eqref{eq: dynamic v} reduces to:
     \begin{align}
        \bvv \cdot \bk + \frac{\|\bvv\|}{\|Dh\|}\frac{\partial h}{\partial t} &=  \frac{\|\bvv\|}{\|Dh\|}Dh \cdot \bk + \frac{\|\bvv\|}{\|Dh\|}\frac{\partial h}{\partial t}\\
        &=  \frac{\|\bvv\|}{\|Dh\|}\left ( Dh \cdot \bk + \frac{\partial h}{\partial t}\right)\\
        &= \frac{\|\bvv\|}{\|Dh\|} \dot{h} \geq 0. 
    \end{align}
    Therefore, by enforcing \eqref{eq: dynamic v}, we equivalently enforce $\dot{h} \geq 0 $ on $\partial \Cc_{\Tc}(t)$, rendering the set $\Cc_{\Tc}(t)$ forward invariant.
\end{proof}

    

Naturally, from this result and in analogy with the static case,  the corresponding activating function  $(\by, t) \mapsto a(\by, t)$ for the dynamic environment is defined as:
\begin{align}\label{eq: activation dynamic}
    a :=  \,\bvv \cdot \bk_{\mathrm{nom}} + \frac{\lVert \bvv \rVert}{\lVert Dh \rVert + \sigma(h)} \cdot \frac{\partial{h}}{\partial{t}}+\gamma h.
\end{align}
The associated activation zones are given by the zero level set of the dynamic activation function \eqref{eq: activation dynamic}. The above results also hold for time-varying guidance fields $(\by, t) \mapsto \bvv(\by, t)$.

\subsection{High Order Systems}
The above results can be extended to high-order systems of the form \eqref{eq: linear output dynamcis1} with relative degree $r\geq 2$ by leveraging CBF backstepping \cite{AJT_backstepping}.
In the first-order case, the safety function $h$ only needs to be continuous. However, for higher-order systems, additional regularity of $h$ is required. To illustrate this, for convenience of exposition, we consider systems
of relative degree $r =2$.
We consider the system $\dot{\bvy} = [\by, \bw]^\top$, where the state is $\bvy = [\by, \dot{\by}]^\top \in \re^4$.
Given a twice continuously differentiable safety function $h \in C^2(\Occ)$ defining the safe set $\Cc$, define:
\begin{align}
    h_\mathrm{B} = h - \frac{1}{2\mu}\|\dot{\by} - \bk_{\bvv} \|^2,
\end{align}
with $\mu > 0$ where $\bk_{\bvv} \in C^2(\Occ;\re^3)$ satisfies $\bvv \cdot \bk_{\bvv} \geq -\gamma h$ for all $\by \in \Cc$ \cite{cohen2024constructive, bahati2025dynamic}. A smooth example, $\bk_{\bvv} \in C^\infty(\Occ;\re^3)$, that remains close to the min-norm activation zones is given in \cite{CohenLCSS23}.
The $0$-superlevel set of$h_\mathrm{B}$ defines the shrunken set:
\begin{equation}
    \Cc_B = \{\, \by \in \mathbb{R}^6 \mid h_\mathrm{B}(\by) \ge 0 \,\} \subset \Cc \times \mathbb{R}^3.
\end{equation}
Since $h_\mathrm{B}(\by) \le h(\by)$ for all $\by \in \Cc$, a guidance field $\bvv$ can be used to ensure that all trajectories starting in $\Cc$ remain in $\Cc$ by rendering $\Cc_\mathrm{B}$ safe. This is possible because $\by \in \Cc_\mathrm{B} \cap \partial \Cc$ implies $\dot{\by} = \bk_{\bvv}$
and under this condition, $\bvv \cdot \bk_{\bvv} \geq -\gamma h$ guarantees that $\bvv(\by)\cdot \dot{\by} > 0$
at all boundary points. 
Specifically, this leads to the following safety constraint:
 \begin{align}
 \bvv \cdot\dot{\by} - 
 \frac{1}{\mu}(\dot{\by} - \bk_{\bvv})  \dot{\bk}_{\bvv}  +
  \frac{1}{\mu}(\dot{\by} - \bk_{\bvv})\bw
 \geq - \gamma h_\mathrm{B}.
\end{align}

Note that $\bk_{\bvv}$ typically depends on $h$, so computing its gradients requires higher-order derivatives of $h$. For this reason, constructing $h$ as a Poisson safety function \eqref{eq: poisson's eq} is advantageous, since its smoothness guarantees the necessary regularity~\cite{bahati2025dynamic}. Leveraging the above result, it follows from \cite[Theorem 4]{AJT_backstepping} that there exists a locally Lipschitz continuous controller rendering $\Cc_\mathrm{B}$ forward invariant. Specifically, if $\by_0 = (\by_0, \dot{\by}_0) \in \Cc_\mathrm{B}$, then $\by(t) \in \Cc_\mathrm{B}$ for all $t \in I_{\max}(\by_0)$. 
Extensions to higher-order systems with relative degree $r>2$ follow by the same construction; see \cite{bahati2025dynamic}.

\begin{algorithm}
\caption{Construct Risk-Aware Safety Filter} \label{alg:main}
    \small
    \KwIn{$\Omega, \partial \Omega,  \mathcal{F}, \mathcal{P}, w, \Phi, d, f, \mathbf{k}_\mathrm{nom}, \gamma$} 
    
    $\partial \Omega_d, \Omega_d \gets \textup{discretize}(\partial \Omega, \Omega, d$)


       $\mathcal{Z} \gets \{ (\mb{y}, f(\mb{y}))\} \text{ for each } \by \in \Omega_d$

       $\mathcal{Y} \gets \{\} $  
    
    \For{$\mb{y} \in \partial\Omega_d$}{
    
        $\textup{feature} \gets \mathcal{F}(\mb{y})$

        $\textup{priority} \gets \mathcal{P}(\textup{feature})$

        $\textup{risk} \gets w(\textup{priority})$

        $b \gets \Phi(w)$

        $\mathcal{Y}\gets \mathcal{Y} \cup \{ (\mb{y}, b)\} $

    }

    $h \gets \textup{DirichletForPoisson}(\Omega_d, \partial \Omega_d, \mathcal{Z})$ \tcp{As in \eqref{eq: poisson's eq}}

    $\vec{\mathbf{v}} \gets \textup{DirichletForLaplace}(\Omega_d, \partial \Omega_d, \mathcal{Y})$ \tcp{As in \eqref{eq: guidance field}}

    $\mathbf{k}_\textup{QP} \gets \textup{BuildSafetyFilter}(\mathbf{k}_\mathrm{nom}, \gamma, h,  \vec{\mathbf{v}})$ \tcp{As in \eqref{eq: safety filter single int constraint}}

    \KwOut{$\mathbf{k}_\textup{QP}$}
    
\end{algorithm}
\vspace{-4mm}

\section{Encoding Variable Conservatism via Boundary Flux}

In safety-critical applications, different obstacles typically demand different levels of caution. 
The \emph{boundary flux} of a guidance field provides a principled way to encode such priorities directly on obstacle surfaces. 
In particular, the boundary flux $b$ quantifies how strongly the guidance field $\bvv$ points outward along the surface normal, $\vec{\mb{v}}= b \hat{\mb{n}}$, providing a way to encode desired levels of caution around an obstacle:
\[
b(\by) \longrightarrow \text{Level of caution at point } \by \in \pOc.
\]
More concretely, by prescribing boundary flux values $b$, we control how activation zones expand or contract around obstacles:
large magnitudes create wider activation zones, while smaller magnitudes yield thinner zones as in Fig. \ref{fig: activation_zones}. 

To make this process constructive, we begin by using a function $\mathcal{F}: \partial \Omega \to \mathcal{L}$ to associate each obstacle position $\mb{y} \in \partial \Omega$ with a \textit{feature} from a set of labels $\mathcal{L}$ that characterizes a desired property of an obstacle or part of an obstacle (e.g., whether the obstacle is a human or a wall). 
We then use a priority function $\mathcal{P} : \mathcal{L} \rightarrow \mathbb{R}_{\geq 0}$ to order the set of labels $\mathcal{L}$, incorporating a user-defined \emph{priority ranking}, specifying which obstacles (or parts of obstacles) demand more caution than others according to the user's preference.
The priorities are then converted to ``risk'' values, which are used to design the flux of the guidance field, yielding the procedure: 
\begin{align}
    \text{State} \; \mapsto \; \text{Feature} \;\mapsto \; \text{Priority} \;\mapsto\; \text{Assigned Risk} \;\mapsto\; \text{Flux}. \nonumber
\end{align}

To ensure that the assigned risks lie on a consistent scale, we normalize them to lie within the interval $[0,1]$ with a risk-assignment map $w: \R_{\geq 0 } \to [0,1]$.
If the range of $\Pc$ is bounded (e.g., probabilities), these values can directly define risk; for example, by rescaling or using the identity map $w(p) = p$ if $p \in[0,1]$. However, if the range of $\Pc$ is not bounded (e.g., when priority is induced by obstacle speed), we use a smooth, bounded, monotonic risk-assignment $w$ to induce the $[0,1]$ bound. Finally, the bounded risk is then mapped to flux values via the function $\Phi: [0,1] \to \R_{<0}$.

\subsection{Constructing Risk-Aware Safety Filters}

\begin{figure}[t!]
    \centering   \includegraphics[width=1.0\linewidth]{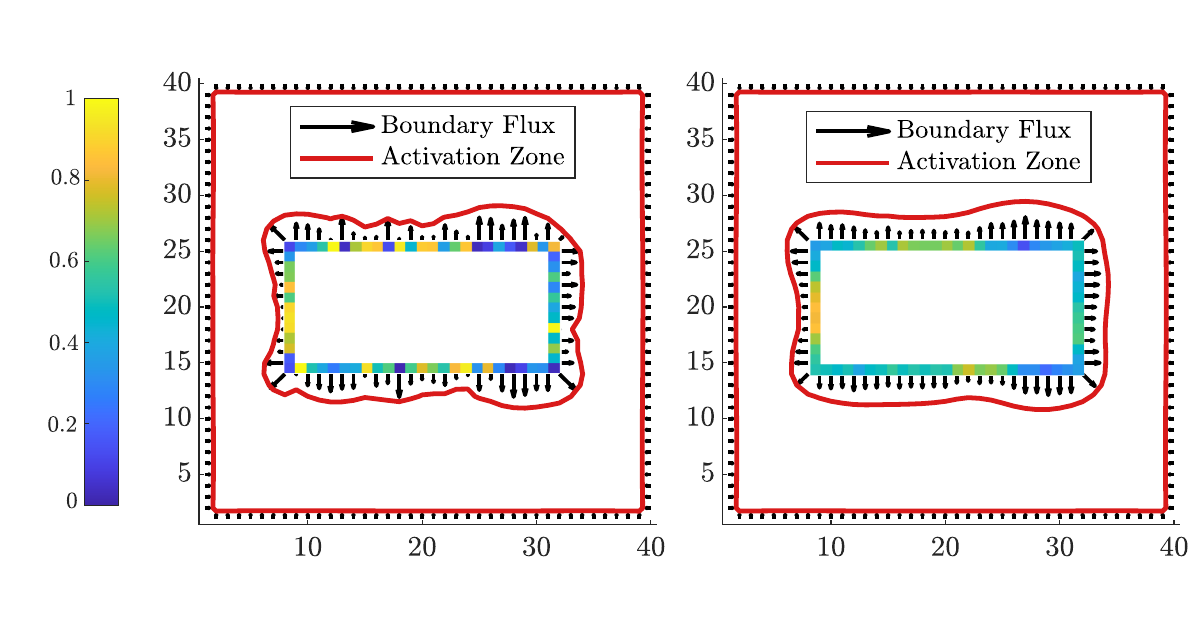}
   \vspace{-10mm}
    \caption{\small{Activation zones based on obstacle uncertainty. Left: Without smoothing, boundary fluxes vary irregularly. Right: Smoothing forces spatial regularity. High confidence regions (yellow) yield tighter activation zones, while low confidence regions (blue) yield expanded activation zones, reflecting the risk prioritization. }}
    \label{fig: location}
    \vspace{-4mm}
\end{figure}

Next, we provide  Algorithm \ref{alg:main} for constructing risk-aware safety filters from environment descriptions, priority rules, and risk assignments.
Using this algorithm requires the following additional inputs: $d\in \R_{>0}$, the discretization resolution size for solving the Poisson and Laplace problems; $f$, the forcing function used when solving for the Poisson safety function $h$ in \eqref{eq: poisson's eq}; $\mb{k}_\textup{nom}: \R^n \to \R^m$, the potentially unsafe nominal control action; and $\gamma\in \R_{>0}$, the scalar value used in the safety constraint as in \eqref{eq: safety filter single int constraint}. The output of Algorithm \ref{alg:main} is a controller, $\mb{k}_\textup{QP}$, as in \eqref{eq: risk-aware safety filter} that guarantees safety and has tunable activation zones that are a result of the risk assigned to the features in the environment. 


\begin{figure*}
    \centering
    \includegraphics[width=1.0\linewidth]{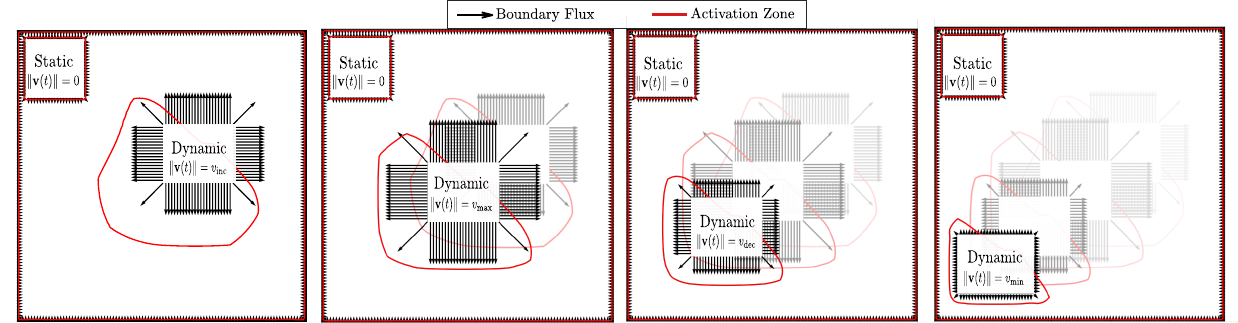}
    \caption{\small{Time-lapse of the activation zones with a moving object in the environment, defined by the zero level set of the function \eqref{eq: activation dynamic}. The object's velocity increases to a maximum before decreasing to a minimum. Larger velocities yield higher boundary fluxes and hence larger activation zones. These zones extend in the direction of motion and are affected by the object's acceleration; they expand during acceleration and contract during deceleration, reflecting the rate of change of the safe set.}}
    \label{fig: velocity}
    \vspace{-4mm}
\end{figure*}

This algorithm is intentionally very general and often simplifies significantly through the choices of inputs. To illustrate this simplicity and the practical utility of this algorithm, 
we provide several examples of how risk-aware safe controllers and their associated activation zones can be generated from relevant environmental features.

\section{Case Studies: Environmental Features}

In this section, we provide example applications of Algorithm \ref{alg:main} to three scenarios: (A) uncertain surface boundaries, (B) dynamic obstacles, and (C) objects in the environment with varying levels of user-assigned semantic priority.

\subsection{Geometric Features: Obstacle Uncertainty}

While location information helps to estimate an obstacle's position, sensors and mapping algorithms introduce significant uncertainty into these estimates. Thus, in this example, we introduce variable conservatism of the safety filter depending on the surface's uncertainty. 

To apply Algorithm \ref{alg:main}, we define the feature map to be the probability of occupancy at the surface location: 
\begin{align}
    \mathcal{F}(\mb{y}) = \mathbb{P}\left(\textup{ occupied at }\mb{y} ~\bigg|~ \sum_{i=1}^{n_{\textup{meas}}} \hat{\mb{m}}_i\right) \in [0,1], 
\end{align}
where $n_\textup{meas}\in \mathbb{N}$ is the number of noisy sensor measurements $\hat{\mb{m}}$. Next we let $\mathcal{P}(p) = 1 - p$ so that low probability (i.e., high uncertainty) are high priority and we let $w$ be the identity function to preserve these risk values. Finally we choose $\Phi$ to linearly interpolate between $b_{\textup{min}}$  and $b_\textup{max}$ so that smaller probabilities lead to larger activation zones and larger probabilities lead to smaller activation zones. The results of this method can be seen in Figure \ref{fig: location} where we see that high-confidence regions (yellow) result in activation regions closer to the surface of $\partial\Omega$ and low confidence regions (blue) result in activation regions that are further away from $\partial \Omega$, leading to more conservatism near
uncertain regions of the obstacle.

\subsection{Dynamic Features: Obstacle Motion}

Motion describes how an obstacle or region evolves over time and introduces additional features such as velocity and acceleration which are continuous and potentially unbounded. These dynamic features may require additional normalization to be incorporated into Algorithm \ref{alg:main}.

First, we define a point's ``feature'' to be the magnitude of its velocity, i.e. $\mathcal{F}(\mb{y}) = \Vert \dot{\mb{y}}\Vert$ for $\mb{y}\in \partial \Omega$. Second, since we prefer our system to be more cautious around high-velocity obstacles, we let the priority function $\mathcal{P}$ be the identity function, meaning that faster objects have higher priority. Next, since $\Vert \dot{\mb{y}}\Vert $ can be unbounded, we normalize this priority value to determine the assigned risk using the function $w(r) = r / (v_\textup{ref} + r)$ so that: 
\begin{align}
    w(\mathcal{P}(\mathcal{F}(\mb{y}))) = \frac{\Vert \dot{\mb{y}}\Vert}{v_\textup{ref} + \Vert \dot{\mb{y}}\Vert } \in [0, 1],
\end{align}
where $v_\textup{ref}\in \R_{>0}$ is a reference speed such that $w(\mathcal{P}(\mathcal{F}(\mb{y}))) = 0.5$ when the speed of the obstacle at $\mb{y}\in \partial\Omega$ is $\Vert\dot{\mb{y}}\Vert = v_\textup{ref}$. Finally, we again define $\Phi$ as a function that linearly interpolates between $b_\textup{min}$ and $b_\textup{max}$ using the risk value. 
The results of this method can be seen in Figure \ref{fig: velocity} which 
 shows a time-lapse of a moving obstacle and the evolution of the flux values and dynamic activation zones.  Here we see that dynamic obstacles are surrounded by larger activation zones than static obstacles, and the activation zones grow uniformly with the obstacle's speed. The activation zones are aligned with the direction of motion and expand in this direction as it accelerates.

\subsection{Semantic Features: Object Type and Priority} 

\begin{figure*}[htbp]
  \centering
  \begin{subfigure}[t]{0.32\textwidth}
    \centering
    \includegraphics[width=\linewidth]{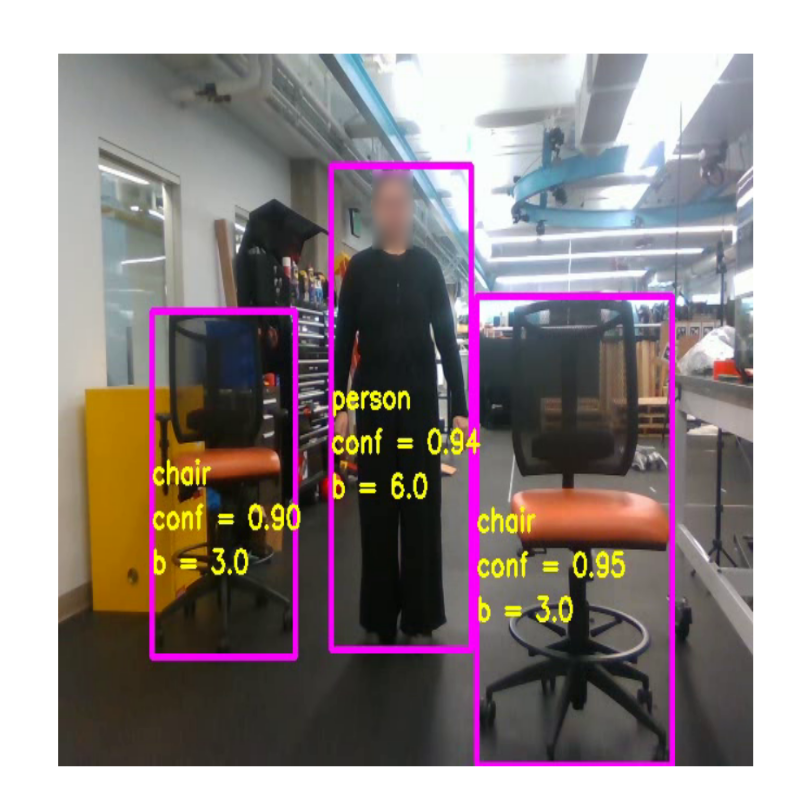}
    \vspace{-6mm}
    \caption{YOLO Semantic Segmentation.}
    \label{fig:semantic_camera_data}
  \end{subfigure}\hfill
  \begin{subfigure}[t]{0.32\textwidth}
    \centering
    \includegraphics[width=\linewidth]{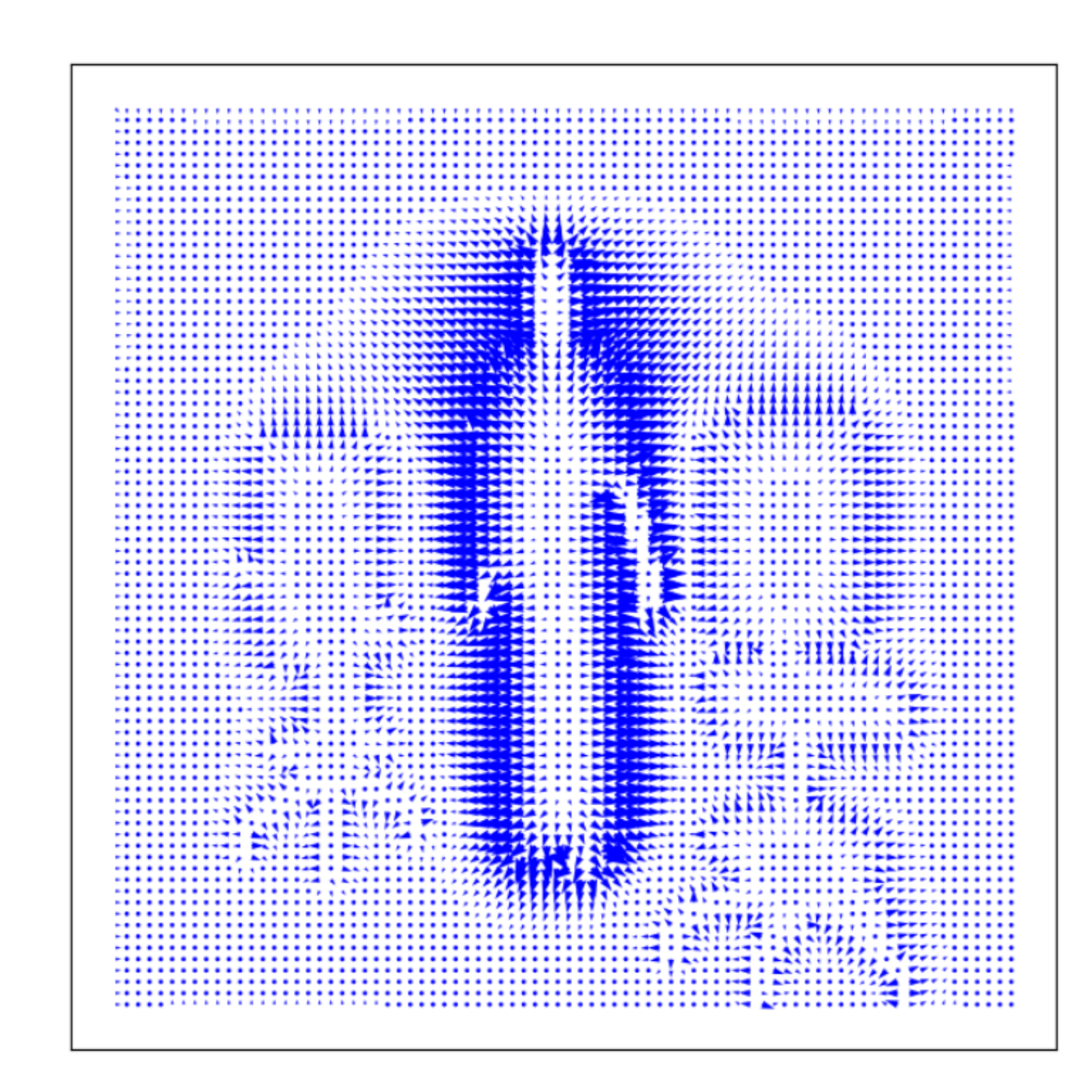}
    \vspace{-6mm}
    \caption{Guidance Field generation.}
    \label{fig:semantic_guidance}
  \end{subfigure}\hfill
  \begin{subfigure}[t]{0.32\textwidth}
    \centering
    \includegraphics[width=\linewidth]{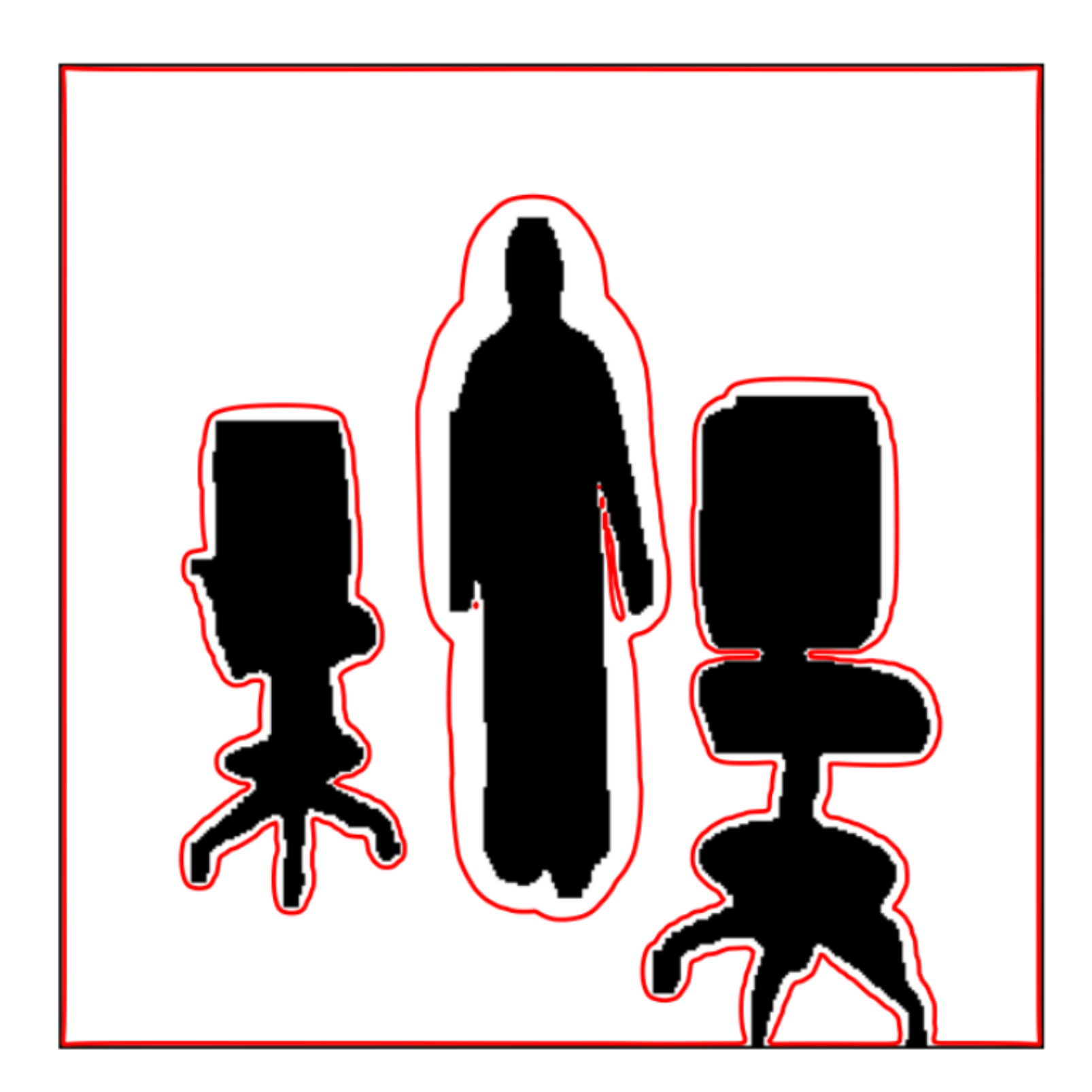}
    \vspace{-6mm}
    \caption{Activation zones.}
    \label{fig:semantic_activation}
  \end{subfigure} 
  \caption{\small{An occupancy map is obtained from semantic segmentation of perception data using a pretrained YOLOv11 model \cite{Jocher_Ultralytics_YOLO_2023} with user designated risk values, assigned according to label classes as by Algorithm \ref{alg:main}: $b_{\text{wall}} = 1,   b_{\text{chair}} = 3$ and $ b_{\text{person}} = 6$. The occupancy map and flux information are then used to generate the guidance field and derive activation zones.}}
  \label{fig:semantic}
  \vspace{-4mm}
\end{figure*}

Next, we extend beyond numerical features and consider the case where we require varying levels of conservatism for each obstacle depending on \textit{what kind} of obstacle it is. Semantic features capture \textit{what} an obstacle is, rather than \textit{where} it is or\textit{ how} it moves. Unlike geometric and dynamic features which can be measured directly, semantic features depend on classification, where obstacles are assigned to discrete categories or labels.

For this example, we consider the discrete set of obstacle labels $\mathcal{L} = \{\textup{wall}, \textup{chair}, \textup{human} \}$, where we obtain the label from a point $\mb{y} \in \partial\Omega$ using $\mathcal{F}$ as $\mb{y} \mapsto \mathcal{F}(\mb{y}) \in \mathcal{L}$. Next, we assign the priority of these semantic labels to be $\mathcal{P}(\textup{wall}) = 1$, $\mathcal{P}(\textup{chair}) = 3$, and $\mathcal{P}(\textup{human}) = 6$. We then assign the risk value using the exponential function $w(r) = 1 - e^{-\alpha r}$ for $\alpha>0$ which maps to $[0,1]$ and applies a high risk to obstacles based on assigned priority. 
Finally, we again define $\Phi$ to be a linear interpolation between the minimum and maximum flux values $b_{\min}$ and $b_{\max}$ associating low risk with low flux and high risk with high flux. 

Figure 5 shows the result of applying Algorithm \ref{alg:main} in this fashion. There, the guidance fields and activation zones around the human are much larger than around the chair and wall obstacles, consistent with the assigned priorities. 

\begin{remark}[Combining features]
Features can also be fused to capture richer context depending on design objectives, provided that the combined mapping preserves the intended priority order before flux assignment.
\end{remark}

\section{Conclusion}
We presented a method for synthesizing risk-aware safety filters using Poisson safety functions and Laplace guidance fields. By adjusting boundary flux, the approach yields tunable boundary gradients around obstacle surfaces, enabling tunable activation zones and conservatism associated with risk. We demonstrated risk-aware safety across diverse contexts; environments with probabilistic occupancies, dynamic
obstacles, and semantic 
labels. Future directions include extensions to 
 general classes of nonlinear systems with closed-loop hardware implementations in real-world settings.


\bibliographystyle{ieeetr}
\bibliography{mainbib, main-GB, cohen}

\begin{thebibliography}{10}

\bibitem{AmesTAC17}
A.~D. Ames, X.~Xu, J.~W. Grizzle, and P.~Tabuada, ``Control barrier function based quadratic programs for safety critical systems,'' {\em IEEE Trans. Autom. Control}, vol.~62, no.~8, pp.~3861--3876, 2017.

\bibitem{AA-AJT-CRH-GO-AA:22}
A.~Alan, A.~J. Taylor, C.~R. He, G.~Orosz, and A.~D. Ames, ``Safe controller synthesis with tunable input-to-state safe control barrier functions,'' {\em IEEE Control Systems Letters}, vol.~6, pp.~908--913, 2022.

\bibitem{long2021learning}
K.~Long, C.~Qian, J.~Cort{\'e}s, and N.~Atanasov, ``Learning barrier functions with memory for robust safe navigation,'' {\em IEEE Robotics and Automation Letters}, vol.~6, no.~3, pp.~4931--4938, 2021.

\bibitem{AmesECC19}
A.~D. Ames, S.~Coogan, M.~Egerstedt, G.~Notomista, K.~Sreenath, and P.~Tabuada, ``Control barrier functions: theory and applications,'' in {\em Proc. Eur. Control Conf.}, pp.~3420--3431, 2019.

\bibitem{TM-RC-AS-WU-AA:22}
T.~G. Molnar, R.~Cosner, A.~Singletary, W.~Ubellacker, and A.~Ames, ``Model-free safety-critical control for robotic systems,'' {\em IEEE Robotics and Automation Letters}, vol.~7, no.~2, pp.~944--951, 2022.

\bibitem{cosner2022safety}
R.~Cosner, M.~Tucker, A.~Taylor, K.~Li, T.~Molnar, W.~Ubelacker, A.~Alan, G.~Orosz, Y.~Yue, and A.~Ames, ``Safety-aware preference-based learning for safety-critical control,'' in {\em Learning for Dynamics and Control Conference}, pp.~1020--1033, PMLR, 2022.

\bibitem{glotfelter2018boolean}
P.~Glotfelter, J.~Cort{\'e}s, and M.~Egerstedt, ``Boolean composability of constraints and control synthesis for multi-robot systems via nonsmooth control barrier functions,'' in {\em 2018 IEEE Conference on Control Technology and Applications (CCTA)}, pp.~897--902, IEEE, 2018.

\bibitem{choi2021robust}
J.~J. Choi, D.~Lee, K.~Sreenath, C.~J. Tomlin, and S.~L. Herbert, ``Robust control barrier--value functions for safety-critical control,'' in {\em 2021 60th IEEE Conference on Decision and Control (CDC)}, pp.~6814--6821, IEEE, 2021.

\bibitem{qin2021learning}
Z.~Qin, K.~Zhang, Y.~Chen, J.~Chen, and C.~Fan, ``Learning safe multi-agent control with decentralized neural barrier certificates,'' {\em arXiv preprint arXiv:2101.05436}, 2021.

\bibitem{robey2020learning}
A.~Robey, H.~Hu, L.~Lindemann, H.~Zhang, D.~V. Dimarogonas, S.~Tu, and N.~Matni, ``Learning control barrier functions from expert demonstrations,'' in {\em 2020 59th IEEE Conference on Decision and Control (CDC)}, pp.~3717--3724, Ieee, 2020.

\bibitem{bahati2025dynamic}
G.~Bahati, R.~M. Bena, and A.~D. Ames, ``Dynamic safety in complex environments: Synthesizing safety filters with poisson’s equation,'' in {\em Robotics: Science and Systems}, 2025.

\bibitem{bahatisafe}
G.~Bahati and A.~D. Ames, ``Safe set synthesis with tunable boundary gradients via poisson safety functions,'' {\em IEEE International Conference on Robotics and Automation (ICRA): Workshop on Robot safety under uncertainty from intangible specifications}, 2025.

\bibitem{ADA-XX-JWG-PT:17}
A.~Ames, X.~Xu, J.~Grizzle, and P.~Tabuada, ``Control barrier function based quadratic programs for safety critical systems,'' vol.~62, no.~8, pp.~3861--3876, 2017.

\bibitem{AnilLCSS22}
A.~Alan, A.~J. Taylor, C.~R. He, G.~Orosz, and A.~D. Ames, ``Safe controller synthesis with tunable input-to-state safe control barrier functions,'' {\em IEEE Contr. Syst. Lett.}, vol.~6, pp.~908--913, 2022.

\bibitem{cosner2023learning}
R.~K. Cosner, Y.~Chen, K.~Leung, and M.~Pavone, ``Learning responsibility allocations for safe human-robot interaction with applications to autonomous driving,'' {\em arXiv preprint arXiv:2303.03504}, 2023.

\bibitem{akella2024risk}
P.~Akella, A.~Dixit, M.~Ahmadi, L.~Lindemann, M.~P. Chapman, G.~J. Pappas, A.~D. Ames, and J.~W. Burdick, ``Risk-aware robotics: Tail risk measures in planning, control, and verification [focus on education],'' {\em IEEE Control Systems}, vol.~45, no.~4, pp.~46--78, 2025.

\bibitem{isidori1985nonlinear}
A.~Isidori, {\em Nonlinear control systems: an introduction}.
\newblock Springer, 1985.

\bibitem{cohen2024constructive}
M.~H. Cohen, R.~K. Cosner, and A.~D. Ames, ``Constructive safety-critical control: Synthesizing control barrier functions for partially feedback linearizable systems,'' {\em IEEE Control Systems Letters}, 2024.

\bibitem{bahati2025control}
G.~Bahati, R.~K. Cosner, M.~H. Cohen, R.~M. Bena, and A.~D. Ames, ``Control barrier function synthesis for nonlinear systems with dual relative degree,'' {\em 2025 IEEE 64st Conference on Decision and Control (CDC)}, 2025.

\bibitem{AndrewCDC22}
A.~J. Taylor, P.~Ong, T.~G. Molnar, and A.~D. Ames, ``Safe backstepping with control barrier functions,'' in {\em Proc. Conf. Decis. Control}, pp.~5775--5782, 2022.

\bibitem{MurrayACC20}
L.~Doeser, P.~Nilsson, A.~D. Ames, and R.~M. Murray, ``Invariant sets for integrators and quadrotor obstacle avoidance,'' in {\em Proceedings of the American Control Conference}, pp.~3814--3821, 2020.

\bibitem{gilbarg1977elliptic}
D.~Gilbarg and N.~S. Trudinger, {\em Elliptic partial differential equations of second order}, vol.~224.
\newblock Springer, 1977.

\bibitem{protter2012maximum}
M.~H. Protter and H.~F. Weinberger, {\em Maximum principles in differential equations}.
\newblock Springer Science \& Business Media, 2012.

\bibitem{marsden2003vector}
J.~E. Marsden and A.~Tromba, {\em Vector calculus}.
\newblock Macmillan, 2003.

\bibitem{CohenLCSS23}
M.~H. Cohen, P.~Ong, G.~Bahati, and A.~D. Ames, ``Characterizing smooth safety filters via the implicit function theorem,'' {\em IEEE Contr. Syst. Lett.}, vol.~7, pp.~3890--3895, 2023.

\bibitem{bena2025geometry}
R.~M. Bena, G.~Bahati, B.~Werner, R.~K. Cosner, L.~Yang, and A.~D. Ames, ``Geometry-aware predictive safety filters on humanoids: From poisson safety functions to cbf constrained mpc,'' {\em IEEE-RAS 24th International Conference on Humanoid Robots (Humanoids)}, 2025.

\bibitem{AJT_backstepping}
A.~J. Taylor, P.~Ong, T.~G. Molnar, and A.~D. Ames, ``Safe backstepping with control barrier functions,'' in {\em 2022 IEEE 61st Conference on Decision and Control (CDC)}, pp.~5775--5782, 2022.

\bibitem{Jocher_Ultralytics_YOLO_2023}
G.~Jocher, J.~Qiu, and A.~Chaurasia, ``{Ultralytics YOLO},'' Jan. 2023.

\end{thebibliography}

\end{document}